\documentclass{article}
\usepackage{microtype}
\usepackage{graphicx} % Required for inserting images
\usepackage{amsthm, amssymb, amsfonts}
\usepackage{amsmath}
\usepackage{mathtools}
\usepackage{multirow}
\usepackage{hyperref}
\usepackage{booktabs}
\usepackage{colortbl}
\usepackage{subcaption}
\usepackage{subcaption}
\usepackage{tabularx}
\usepackage[preprint]{tmlr}
\usepackage{natbib}
\usepackage{xcolor}
\usepackage[shortlabels]{enumitem}
\usepackage{soul}

\usepackage[linesnumbered, ruled]{algorithm2e}

\newcommand{\IT}{AHSC}
\newcommand{\gray}[1]{\cellcolor{gray!20}\textbf{#1}}

\DeclareMathOperator*{\argmax}{arg\,max}

\let\oldcite\cite
\renewcommand{\cite}[1]{(\oldcite{#1})}

\newtheorem{theorem}{Theorem}

\newtheorem{lemma}{Lemma}
\newtheorem{definition}{Definition}

\newcommand{\convexity}{$\mu-$strong convexity}

\DeclarePairedDelimiter\abs{\lvert}{\rvert}%
\DeclarePairedDelimiter\norm{\lVert}{\rVert}%
\makeatletter
\let\oldabs\abs
\def\abs{\@ifstar{\oldabs}{\oldabs*}}
\let\oldnorm\norm
\def\norm{\@ifstar{\oldnorm}{\oldnorm*}}
\makeatother

\title{Strong convexity-guided hyper-parameter optimization for flatter losses}

 \author{
     \name Rahul Yedida 
     \email ryedida@ncsu.edu \\
     \addr Department of Computer Science \\ North Carolina State University, USA
     \AND
     \name Snehanshu Saha
     \email snehanshus@goa.bits-pilani.ac.in \\
     \addr Department of CSIS and APPCAIR \\ BITS Pilani K.K Birla Goa Campus\\ HappyMonk AI, Bangalore, India
 }

\begin{document}

\maketitle

\begin{abstract}
    We propose a novel white-box approach to hyper-parameter optimization. Motivated by recent work establishing a relationship between flat minima and generalization, we first establish a relationship between the strong convexity of the loss and its flatness. Based on this, we seek to find hyper-parameter configurations that improve flatness by minimizing the strong convexity of the loss. By using the structure of the underlying neural network, we derive closed-form equations to approximate the strong convexity parameter, and attempt to find hyper-parameters that minimize it in a randomized fashion. Through experiments on 14 classification datasets, we show that our method achieves strong performance at a fraction of the runtime.
\end{abstract}

\section{Introduction}

A typical machine learning pipeline involves using a combination of processes that have hyper-parameters that the analyst sets. There is significant interest in automatically computing a Pareto-optimal set of hyper-parameters tailored to the problem \cite{agrawal2019dodge, cowen2022hebo, li2017hyperband, bergstra2011algorithms, bergstra2012random, falkner2018bohb, eriksson2019scalable, ansel:pact:2014, snoek2012practical, hernandez2014predictive, swersky2014freeze, snoek2015scalable, bergstra2013making}. In parallel, there is a venerable line of work studying the loss landscapes of neural networks \cite{hochreiter1994simplifying, hochreiter1997flat, hinton1993keeping, chaudhari2019entropy, keskar2016large, dziugaite2017computing, mcallester1999pac, neyshabur2014search, neyshabur2017exploring, li2018visualizing, seong2018towards, dauphin2014identifying, choromanska2015loss, zhang2021understanding}. Notably, prior work has shown the effectiveness of improving the smoothness of loss surfaces via batch normalization \cite{santurkar2018does} and filter normalization \cite{li2018visualizing}.

Hyper-parameter optimization (HPO) is well-studied, with the most popular approaches being based on Bayesian optimization \cite{snoek2012practical, hernandez2014predictive, swersky2014freeze, bergstra2013making}, while other work suggests alternative approaches such as random search \cite{bergstra2012random} and tabu search \cite{agrawal2019dodge}. \citet{smith2018disciplined} discusses empirical methods to manually tune hyper-parameters based on the performance of the current system. However, although HPO has repeatedly been shown to improve learner performance \cite{Tantithamthavorn16,majumder2018500+}, much applied machine learning research either does not use HPO, or uses computationally expensive methods such as grid search. Some of this reluctance to use HPO stems from the general view that it is computationally expensive. For example, \citet{tran2020hyper} comment, ``\textit{Regardless of which hyper-parameter optimization method is used, this task is generally very expensive in terms of computational costs.}'' Moreover, there is a growing concern to reduce the carbon emissions from ML experiments \cite{lacoste2019quantifying}. Indeed, NeurIPS now suggests authors to report the carbon emissions from their experiments.

\begin{figure*}
    \centering
    \begin{subfigure}{0.3\textwidth}
        \includegraphics[width=\linewidth]{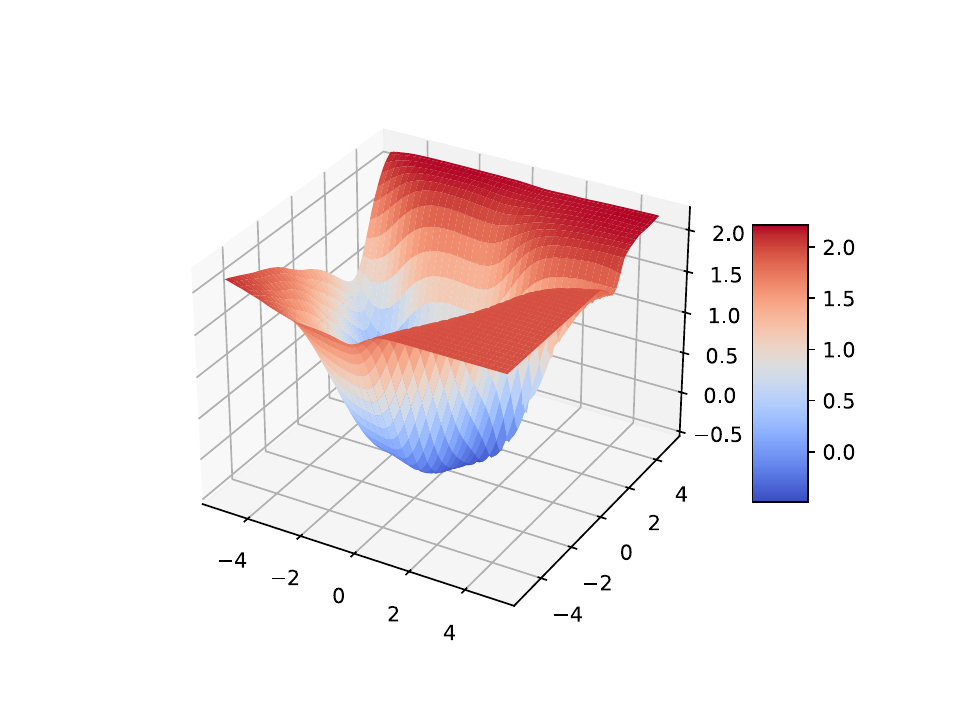}
    \end{subfigure}
    \begin{subfigure}{0.3\textwidth}
        \includegraphics[width=\linewidth]{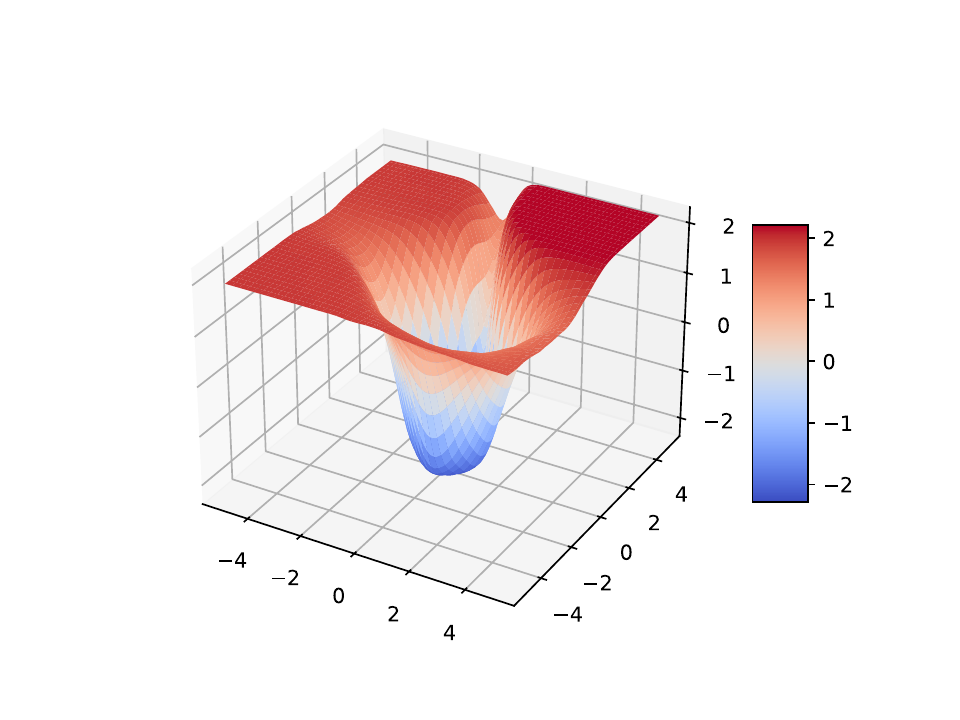}
    \end{subfigure}  
    \begin{subfigure}{0.35\textwidth}
        \includegraphics[width=\linewidth]{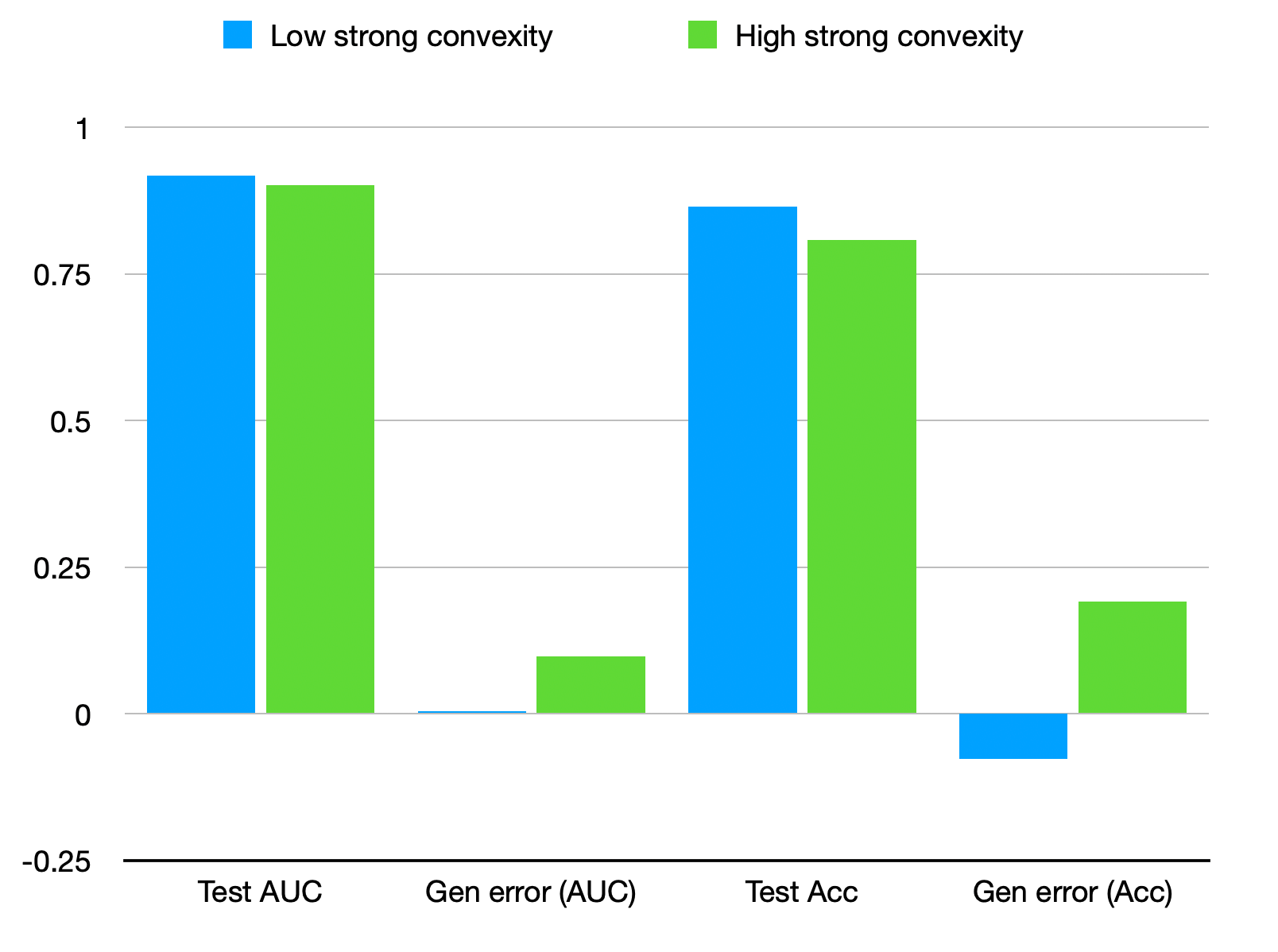}
    \end{subfigure}  
    \caption{Landscapes (plotted using \citet{li2018visualizing}) with their corresponding metrics, on the Australian (binary classification, imbalanced) dataset. \textbf{Left:} a landscape with lower strong convexity (0.112), and consequently, a wider minima. \textbf{Middle:} a landscape with high strong convexity (1.133), which leads to a sharp minima. \textbf{Right:} Test metrics and generalization error for the two hyper-parameter configurations. Although the sharper configuration converged faster to a training error of 0, it generalizes poorly and performs worse on the test set.}
    \label{fig:landscape}
\end{figure*}

Motivated by the need for computationally cheaper HPO methods, we pose the following question: \textit{can we aim to directly improve the desirable properties of loss landscapes by exploiting the structure of the learning algorithm?} Specifically, recent work has repeatedly endorsed the relationship between the \textit{flatness} of local minima and generalization ability of networks \cite{keskar2016large, jiang2019fantastic, neyshabur2017exploring, dziugaite2017computing, li2018visualizing, jastrzkebski2017three}. We use four major advances in the theoretical understanding of loss landscapes: (i) \citet{wu2023implicit} show that SGD can escape from low-loss, sharp minima (measured by the Frobenius norm of the Hessian) exponentially fast; (ii) \citet{dauphin2014identifying} used the line of work starting with \citet{bray2007statistics} to show that saddle points are exponentially more likely than local minima; (iii) gradient descent dynamics repel from saddle points (iv) the sharpness measure proposed by \citet{keskar2016large} have been repeatedly endorsed to correlate well with generalization \cite{jiang2019fantastic}.

In this work, we show that minimizing the supremum of the sharpness is equivalent in formulation to computing the infimum of the strong convexity of the loss in a mini-batch fashion. Next, we demonstrate a semi-empirical method of computing the strong convexity of a loss function parameterized by the hyper-parameters of the model. The result we obtain is general enough to cover a wide range of network topologies. We use this result to motivate a hyper-parameter optimization method that uses the strong convexity as a heuristic for search. Our method requires fewer full-length training runs of the learning algorithm, instead relying on one-epoch cycles to compute the strong convexity, and discarding hyper-parameter configurations that are not promising. 

Figure \ref{fig:landscape} shows the motivation for our approach. The left side shows a landscape with lower strong convexity (and a flatter minima), which in turn has much lower generalization errors for both accuracy and AUC; the middle shows a landscape with a higher strong convexity (and a sharper minima), which led to a much higher generalization error for both metrics. Note that for the latter case, the training stopped early since both training AUC and accuracy reached 1; however, the model generalized poorly, and did worse on the test set.

% In particular, we experiment with directly minimizing the strong convexity of the loss landscape, and show that it is an effective method of hyper-parameter optimization. We first motivate the approach by making the connection to flatness of minima, which has been extensively studied (see Section \ref{sec:background} for a review). In particular, it has been repeatedly shown that flatter minima correspond to better generalization. This paper shows that minimizing the strong convexity flattens the local minima of a function. We then derive equations for the strong convexity of feedforward networks, which applies to more general topologies. Finally, we experimentally demonstrate the strength of our approach by comparing it to several state-of-the-art hyper-parameter optimization algorithms across a variety of datasets. We match or outperform other algorithms in XXX\% of datasets, while only needing XXX\% of the computational cost.

Our contributions are as follows:
\begin{itemize}
    \item We propose a novel white-box hyper-parameter optimization algorithm based on minimizing the strong convexity of the loss.
    \item We make the theoretical connection between the flatness of losses and strong convexity, motivating our approach.
    \item We show that our algorithm achieves strong performance in HPO across 14 datasets at a fraction of the computational cost.
\end{itemize}

To allow others to reproduce our work, our code is available online\footnote{
%\ifanon
%\url{https://anonymous.4open.science/r/ahsc-hpo/}
%\else 
\url{https://github.com/yrahul3910/strong-convexity}
%\fi
}.

\section{Related Work}

This section briefly discusses related work; for a more comprehensive discussion, please see Appendix \ref{sec:app:lit}.

\textbf{Hyper-parameter optimization.} In its general form, hyper-parameter optimization (HPO) solves the problem of finding a non-dominated hyper-parameter configuration under some budget. Early works \cite{bergstra2012random, bergstra2011algorithms} showed the strength of random search, but since then, Bayesian Optimization has become increasingly popular. For example, the Tree of Parzen Estimators (TPE) algorithm models $p(x|y)$ using two kernel density estimates depending on whether $y$ is below or above some quantile, and optimizes the Expected Improvement (EI). 

HyperBand (HB) \cite{li2017hyperband} uses a procedure called ``successive halving", which starts by randomly sampling a set of configurations and testing them under a limited budget, retaining only the best-performing ones and allocating those greater resources. At its core, its strategy is to aggressively prune poor-performing configurations so that more promising ones can be allocated more resources. Algorithms such as BOHB \cite{falkner2018bohb} and DEHB \cite{awad2021dehb} improve upon these in different ways: DEHB uses a distributed computing approach and combines differential evolution with HB, while BOHB combines a slightly modified version of the BO-based TPE with HB.

HEBO \cite{cowen2022hebo} use a combination of input and output transformations along with NSGA-II to optimize a multi-objective acquisition function. TuRBO \cite{dou2023turbo} assumes the hyper-parameter to performance mapping is Lipschitz, and uses an ensemble of learners to predict performance, using the prediction to update its Gaussian Process model instead if that prediction is poor.

\textbf{Flat minima and generalization.} The connection between ``flat minima'' and generalization has been repeatedly endorsed. The flatness of minima has been defined in various ways, such as the volume of hypercuboids such that the loss is within a tolerance \cite{hochreiter1997flat} and as robustness to adversarial perturbations in weight space \cite{keskar2016large}. The specific formulation of \citet{keskar2016large} is
$$
\zeta(\boldsymbol w; \epsilon) = \frac{\max\limits_{\boldsymbol v \in \mathcal{B}(\epsilon, \boldsymbol w)} f(\boldsymbol v) - f(\boldsymbol w)}{1 + f(\boldsymbol w)}
$$
This notion of sharpness was also endorsed by a large-scale study of complexity measures by \citet{jiang2019fantastic}. The relationship between flatness and generalization was later also endorsed by several works \cite{neyshabur2017exploring, li2018visualizing, wu2023implicit}. We defer the reader to Appendix \ref{sec:app:lit} for a more detailed review.

\section{Method}

\subsection{Notation and Assumptions}

For any learner, $X \in \mathcal{X}$ will represent the independent variables and $y$ will represent the labels; $W$ represents the weights of a neural network. We use $m$ to denote the number of training samples, $n$ to denote the number of features, $k$ to denote the number of classes, and $E$ to denote the loss function.

For a feedforward network, $L$ represents the number of layers, and at each layer, the following computation is performed:
       $ z^{[l]} = W^{[l]T}a^{[l-1]} + b^{[l]} \label{eq:zl}$
       and
        $a^{[l]}  = g^{[l]}(z^{[l]})$ where
        $a^{[0]} = X$.
Here, $b^{[l]}$ is the bias vector at layer $l$, $W^{[l]}$ is the weight matrix at layer $l$, and $g^{[l]}$ is the activation function used at layer $l$--we will assume this is the ReLU function $g^{[l]}(x) = \max(0, x)\  \forall l \in \{1,\ldots,L-1\}$. 

We use $\nabla_W \cdot$ to denote the first gradient of $\cdot$, with respect to $W$. $\nabla^2_W \cdot$ denotes the Hessian. Finally, $\mathcal{B}(\epsilon, w)$ denotes an $\epsilon-$ball centered at $w$. Whenever unspecified, the matrix norm is the Frobenius norm. We use $\mathbb{S}^n$ to denote the set of $n\times n$ real, symmetric matrices. We defer to Appendix \ref{sec:app:background} for background definitions in convex optimization.

We assume the space $\mathcal{X}$ is Polish\footnote{A Polish space is a complete, separable, metrizable space.}. We use the Frobenius norm of the Hessian to define the strong convexity of a loss landscape. In this context, completeness of the underlying metric space is necessary. Moreover, completeness is an important assumption in gradient descent, since it guarantees the existence of limits for Cauchy sequences.  Finally, separability helps avoid issues with measurability when we use a covering argument in the next section. 
%Measurability and other useful properties may not hold for non-separable spaces such as Skorokhod spaces (see \cite{van1996weak}).

% \begin{definition}
% 	A function $f(x)$ is $\mu-$strongly convex if for $\mu > 0, \forall x \in \text{dom } f$, 
% 	$$
% 	f(x) - \frac{\mu}{2} \lVert x \rVert^2
% 	$$
% 	is convex.
% \end{definition}

% Strong convexity provides a lower bound for the function’s curvature. 

\subsection{{\IT}: Accelerated HPO via Strong Convexity}
\label{sec:stcvx}

\begin{lemma}
    Let $A, B \in \mathbb{S}^n$ and suppose $A \succeq B$. Then $\norm{A}_2 \geq \norm{B}_2$
    \label{lemma:positivedef}
\end{lemma}
\begin{proof}
    We have $\norm{A}_2 = \sqrt{\lambda_{max}(A^T A)} = \abs{\lambda_{max}(A)}$ since $A$ is symmetric. 
    Since $A \succeq B$, we have by definition $x^T A x \geq x^T Bx$ for an arbitrary $x \in \mathbb{R}^n$. From the min-max theorem, we have $\forall k \in \{1, \ldots, n\}$,
    \begin{equation*}
    %\scalebox{0.83}{$
    \lambda_k(A) = \min\limits_{\substack{\mathcal{M} \subset \mathbb{S}^n \\ \dim \mathcal{M} = k}} \max\limits_{x \in \mathcal{M} \backslash \{0\}} \frac{x^T A x}{x^T x} \geq \min\limits_{\substack{\mathcal{M} \subset \mathbb{S}^n \\ \dim \mathcal{M} = k}} \max\limits_{x \in \mathcal{M} \backslash \{0\}} \frac{x^T B x}{x^T x} = \lambda_k(B)
    %$}
    \end{equation*}
    where $\lambda_k$ is the $k$th eigenvalue in the spectrum ordered in non-increasing order.
    Using the above with $k=1$ completes the proof.
\end{proof}

We start by writing the definition of strong convexity in terms of positive-semidefiniteness as
$$
\nabla^2 f(x) \succeq \mu I
$$
From Lemma \ref{lemma:positivedef}, we can rewrite this as
$$
\lVert \nabla^2 f(x) \rVert_2 \geq \mu \Rightarrow \norm{\nabla^2 f(x)}_F \geq \mu
$$
and so we redefine $\mu-$strong convexity as 
$$
\mu = \inf \lVert \nabla^2 f(x) \rVert_F
$$
A key motivation for minimizing the strong convexity is to relate it to the flatness of minima. Since $\nabla^2 f$ is symmetric (by Schwarz's theorem), we have
$$
\abs{\lambda_{max}(\nabla^2 f)} = \norm{\nabla^2 f}_2 \leq \norm{\nabla^2 f}_F
$$
so that $\abs{\lambda_{max}(\nabla^2 f)} = \inf \norm{\nabla^2 f}_F$ which directly relates the sharpness of the landscape with the definition of strong convexity above. We can also relate the above definition of strong convexity to the sharpness measure $\zeta$ established by \citet{keskar2016large}:
\begin{align*}
    \zeta(w) &= \frac{\max\limits_{w^\prime \in \mathcal{B}(\epsilon, w)} f(w^\prime) - f(w)}{1 + f(w)} \\
    &\overset{(i)}{\approxeq} \max\limits_{w^\prime \in \mathcal{B}(\epsilon, w)} f(w^\prime) - f(w) \\
    &\overset{(ii)}{\approxeq} \frac{\epsilon^2}{2} \norm{\nabla^2 f(w)}_2 \leq \frac{\epsilon^2}{2} \norm{\nabla^2 f(w)}_F
\end{align*}
where (i) is because the training error is typically small in practice \cite{neyshabur2017exploring} and (ii) follows from a second-order Taylor expansion of $f$ around $w$, as done by  \citet{dinh2017sharp}.  Therefore, we wish to minimize the sharpness, which is equivalent to the formulation above. Equivalently, we have (with $r$ denoting the rank of the Hessian)
\begin{equation*}
%\scalebox{0.85}{$
    \zeta(w) \approxeq \frac{\epsilon^2}{2} \norm{\nabla^2 f(w)}_2 \geq \frac{\epsilon^2}{2\sqrt{r}} \norm{\nabla^2 f(w)}_F \geq \frac{\epsilon^2}{2\sqrt{n}} \norm{\nabla^2 f(w)}_F
%$}
\end{equation*}
so that the strong convexity is a lower bound on the sharpness, and a higher strong convexity implies a higher sharpness, which is correlated with a higher generalization error\footnote{Assuming the network has sufficient capacity.}. Therefore, we have:
$$
\frac{\epsilon^2}{2\sqrt{n}}\norm{\nabla^2 f(w)}_F \leq \zeta(w) \leq \frac{\epsilon^2}{2}\norm{\nabla^2 f(w)}_F
$$
so that the strong convexity and a scaled version thereof provide bounds on the sharpness, and minimizing the strong convexity implies both the lower and upper bounds on the sharpness are lowered.
%\textcolor{red}{Precisely, if $\mu$ is larger then generalization error is high}

Motivated by the above relationship, we seek to find hyper-parameters that minimize this strong convexity. For a learner parameterized by $H \in \mathcal{H}$, denote its loss function as $E(x; H)$. Then, we would like to solve the problem
$$
	\inf\limits_{H \in \mathcal{H}}\sup\limits_{x \subset X} \lVert \nabla^2 E(x; H) \rVert
$$
This mini-batch version accounts for large datasets for which a mini-batch approach is necessary. Note that it is necessary to use an upper bound approximation at the mini-batch level: using an $\inf$ instead yields 0 a majority of the time, making the search ineffective. Indeed, we can relate this to the sharpness measure from \citet{keskar2016large}: this formulation corresponds to minimizing the upper bound on the sharpness over the entire $\epsilon-$ball.

We can bound the deviation from the true supremum using a covering argument \cite{duchi2023}. Let $f(x) = \sup \norm{\nabla^2 E(x)}$ be drawn from some set of functions $\mathcal{F}$, each of whose elements map from $\mathcal{X}$ to $\mathbb{R}$. Define a point-mass empirical distribution on $\{x_i\}_{i=1}^m$ as $P_m = \frac{1}{m} \sum\limits_{i=1}^m \delta_{x_i}$ where $\delta$ is the Dirac delta. For any function $f: \mathcal{X} \to \mathbb{R} \in \mathcal{F}$, let 
$$
P_m f \triangleq \mathbb{E}_{P_m}[f(X)] = \sum\limits_{i=1}^m f(x_i)
$$
be the empirical expectation over a mini-batch and let
$$
Pf \triangleq \mathbb{E}_P[f(X)] = \int f(x) dP(x)
$$
denote the general expectation under a measure $P$. Suppose the functions in $\mathcal{F}$ are bounded above by $\beta$ (trivially, they are bounded below by 0), and define the metric over $\mathcal{F}$ as $\norm{f - g}_\infty = \sup\limits_{x \in \mathcal{X}} \abs{f(x) - g(x)}$. Denote by $N(\delta, \Theta, \rho)$, the covering number for a $\delta-$cover of a set $\Theta$ with respect to a metric $\rho$. Then, we use the standard covering number guarantee (cf. \citet{duchi2023} Ch. 4) to get
\begin{equation*}
\scalebox{0.82}{$
        P\left( \sup\limits_{f \in \mathcal{F}} \abs{P_m f - Pf} \geq t \right) \leq \exp\left( -\frac{mt^2}{18\beta^2} + \log N(t/3, \mathcal{F}, \norm{\cdot}_\infty) \right)
$}
\end{equation*}

Algorithm \ref{alg:stcvx} shows our overall approach. We first sample $N_1$ random configurations (line 2). For each of these configurations, we first train the model for one epoch to bring the weights closer to their final weights (line 6). Training for a single epoch provides a balance between the cost associated with training fully (which would provide a more accurate estimate for the strong convexity), and not training at all (which provides a very poor estimate). We then compute $\norm{\nabla^2 E(x)}$ in a mini-batch fashion (lines 8-10). Since we wish to minimize the strong convexity, it is important that we look at the highest value across mini-batches, and aim to minimize that upper bound. Importantly, if the strong convexity is 0 (implying the function is \textit{not} strongly convex), we discard that configuration (lines 11-13). We pick the configurations corresponding to the $N_2$ lowest values of strong convexity as computed above, and train those models fully (lines 15-17). Finally, we return the best-performing configuration. 

In Appendix \ref{sec:proofs}, Theorem \ref{lemma:smoothgd}, we show that if the loss is smooth and $\mu-$strongly convex, then 
$$
f(x_t) - f(x^*) \leq ( 1 - \alpha \mu )^t \left( f(x_0) - f(x^*) \right)
$$
where $\alpha$ is the learning rate. Therefore,
\begin{equation*}
        \begin{aligned}
            f(x_{t+1}) - f(x_t) &\leq\left( (1 - \alpha \mu)^{t+1} - (1 - \alpha \mu)^t \right) \left( f(x_0) - f(x^*) \right) \\
                &\leq (1 - \alpha \mu)^t (-\alpha \mu) \left( f(x_0) - f(x^*) \right) \\
                &\leq \exp(-\alpha \mu t) (-\alpha \mu) \left( f(x_0) - f(x^*) \right)
        \end{aligned}
\end{equation*}
which implies exponentially decaying benefit as the number of epochs increases. On the other hand, Theorem \ref{lemma:smoothgd} also implies that for vanilla gradient descent, the number of steps required for convergence is inversely proportional to the strong convexity, so that minimizing the latter implies a greater number of steps is required to converge (which increases the runtime). To reduce the impact of this, we use the Adam \cite{kingma2014adam} optimizer. We leave it to future work to explore additional strategies, such as large adaptive learning rates, which can also lead to flatter losses \cite{jastrzkebski2017three}.

\begin{algorithm}[ht]
	\SetAlgoLined
	\SetKwInOut{KwInput}{Input}
    \KwInput{Number of configurations to sample $N_1$, number of configurations to run $N_2$. Defaults: $N_1 = 50, N_2 = 10$}

        $\mathcal{H}_0 \gets$ \textsc{Random}($\mathcal{H}, N_1$)\;
        $S \gets \phi$  \tcp*{Strong convexity values}\;
        $P \gets \phi$ \tcp*{Performance scores}\;

        \For{config $h$ in $\mathcal{H}_0$}{
            Train for one epoch using $h$\;
            $\mu_{max} = -\infty$\;
            \For{mini-batch $x \subset X$}{
                $\mu_{max} = \max(\mu_{max}, \norm{\nabla^2 f(x; h)})$\;
            }
            \If{$\mu_{max} > 0$}{
                S[h] $\gets \mu_{max}$\;
            }
        }

        \For{config $h$ in \textsc{Lowest}(S, $N_2$)}{
            P[h] $\gets$ \textsc{Run}(h)\;
        }

        \KwRet $\argmax P$
    \caption{{\IT}}
    \label{alg:stcvx}
\end{algorithm}

%\subsection{Classification Settings}

We now consider the general multi-class classification problem, where the cross-entropy loss is used, and the last layer of the neural network uses a softmax activation. Below, we show that the strong convexity of a feedforward network with a softmax activation in the last layer, trained on the cross-entropy loss, is given by
\[
    \inf \left\lVert \nabla^2_W E \right\rVert \propto \frac{1}{m} \inf \frac{\lVert a_j^{[L-1]} \rVert}{\lVert W^{[L]} \rVert}
\]

Importantly, the proof does not rely on the architecture of the network beyond the last two layers. That is, as long as the last two layers of the network are fully-connected, this theorem applies.

\begin{lemma}
    \label{lemma:partial}
    For a neural network with ReLU activations in the hidden layers and a softmax activation at the last layer, 
    \[
        \frac{\partial E}{\partial z^{[L]}_j} = \frac{1}{m} \left( \sum\limits_{i=1}^m [y^{(i)} = j] \right) \left( \sum\limits_{h=1}^k a^{[L]}_j - \delta_{hj} \right)
    \]
    under the cross-entropy loss.
\end{lemma}

\begin{proof}
    We will use the chain rule, as follows:
    \begin{align}
    	\frac{\partial E}{\partial z^{[L]}_{j}} &= \frac{\partial E}{\partial a^{[L]}_j}\cdot \frac{\partial a^{[L]}_j}{\partial z^{[L]}_j} \label{eq:int:5}
    \end{align}
    Consider the Iverson notation version of the general cross-entropy loss:
    $$
    E(a^{[L]}) = -\frac{1}{m} \sum\limits_{i=1}^m \sum\limits_{h=1}^k [y^{(i)} = h] \log a_h^{[L]}
    $$
    Then the first part of \eqref{eq:int:5} is trivial to compute:
    \begin{equation}
        \frac{\partial E}{\partial a_j^{[L]}} = -\frac{1}{m} \sum\limits_{i=1}^m \frac{[y^{(i)}=j]}{a_j^{[L]}} \label{eq:mul:2}
    \end{equation}
    The second part is the derivative of the softmax and is equal to
    \begin{align} 
        \frac{\partial a^{[L]}_j}{\partial z^{[L]}_p} &= a^{[L]}_j([p=j]-a^{[L]}_p) \label{eq:mul:3}
    \end{align}

    Combining \eqref{eq:mul:2} and \eqref{eq:mul:3} in \eqref{eq:int:5} gives
    \begin{equation}
        \frac{\partial E}{\partial z^{[L]}_j} = \frac{1}{m} \left( \sum\limits_{i=1}^m [y^{(i)} = j] \right) \left( \sum\limits_{h=1}^k a^{[L]}_j - \delta_{hj} \right)
        \label{eq:mul:4}
    \end{equation}
\end{proof}

\begin{theorem}[{\convexity} for neural classifiers]
    For a neural network with ReLU activations in the hidden layers, a fully connected penultimate layer, and a softmax activation at the last layer, the {\convexity} of the cross-entropy loss is given by
    \begin{equation}
    \label{eq:beta-ff}
        \inf \left\lVert \nabla^2_W E \right\rVert \propto \frac{1}{m} \inf \frac{\lVert a_j^{[L-1]} \rVert}{\lVert W^{[L]} \rVert}
    \end{equation}
    \label{th:stcvx:ff}
\end{theorem}
\begin{proof}
For a neural network with ReLU activations, Lemma \ref{lemma:partial} gives us (using the chain rule one step further):
\[
    \nabla_{W^{[L]}_j} E = \frac{1}{m} \left( \sum\limits_{i=1}^m [y^{(i)} = j] \right) \left( \sum\limits_{h=1}^k a^{[L]}_j - \delta_{hj} \right) a_j^{[L-1]}
\]
Therefore,
\begin{equation*}
%\scalebox{0.8}{$
\begin{aligned}
    \nabla^2_{W^{[L]}_j} E &= \frac{1}{m} \left( \sum\limits_{i=1}^m [y^{(i)} = j] \right) \left( \sum\limits_{h=1}^k a^{[L]}_j - \delta_{hj} \right) \nabla_{W^{[L]}_j} a_j^{[L-1]} \\
     &= \frac{1}{m} \left( \sum\limits_{i=1}^m [y^{(i)} = j] \right) \left( \sum\limits_{h=1}^k a^{[L]}_j - \delta_{hj} \right) \dfrac{\partial a_j^{[L-1]}}{\partial E} \dfrac{\partial E}{\partial W^{[L]}_{j}} & \\
     &= \frac{1}{m^2} \left( \sum\limits_{i=1}^m [y^{(i)} = j] \right)^2 \left( \sum\limits_{h=1}^k a_j^{[L]} - [h = j] \right)^2 \frac{\partial a_j^{[L-1]}}{\partial E} a_j^{[L-1]} \\ 
     &= \frac{1}{m^2} \left( \sum\limits_{i=1}^m [y^{(i)} = j] \right)^2 \left( \sum\limits_{h=1}^k a_j^{[L]} - [h = j] \right)^2 \frac{1}{\dfrac{\partial E}{\partial a_j^{[L-1]}}} a_j^{[L-1]}  & \\
\end{aligned}
%$}
\end{equation*}

Using $z^{[l]} = W^{[l]T}a^{[l-1]} + b^{[l]}$ and Lemma \ref{lemma:partial} gives us:
\[
\begin{aligned}
    \dfrac{\partial E}{\partial a_j^{[L-1]}} &= \dfrac{\partial E}{\partial z_j^{[L]}} \dfrac{\partial z_j^{[L]}}{\partial a_j^{[L-1]}} \\
    &= \frac{1}{m} \left( \sum\limits_{i=1}^m [y^{(i)} = j] \right) \left( \sum\limits_{h=1}^k a_j^{[L]} - [h = j] \right) W_j^{[L]}
\end{aligned}
\]
The limiting case of this is when all softmax values are 0, so that the second summation term is -1. The first summation is a positive constant, which we drop as a proportionality constant. Now, \eqref{eq:beta-ff} immediately follows.
\end{proof}

\section{Experiments}
\label{sec:experiments}

\begin{table*}[ht!]
{\scriptsize
\centering
\caption{Experimental results on various classification datasets. {\IT} is our method. Values shown are medians over 20 repeats. Statistically best results are highlighted in \textbf{bold} (see Section \ref{sec:experiments} for details).}
\label{tab:results}
    \begin{subtable}[t]{.45\linewidth}
        \begin{tabular}[t]{llr}
            \toprule
            \multicolumn{3}{c}{\textbf{Image}} \\
            \midrule
            \textbf{Dataset}                 & \textbf{HPO method} & \multicolumn{1}{l}{\textbf{Accuracy}} \\
            \midrule
            \multirow{6}{*}{MNIST} & {\IT}                & \gray{98.65}                    \\
                                    & Hyperopt            & 97.22                     \\
                                    %& Opentuner           & \gray{99.02}                     \\
                                    & Random       & \gray{98.95}                              \\
                                    & TuRBO               & \gray{98.99}                              \\
                                    & HEBO                & \gray{98.94}                              \\
                                    & BOHB                & 98.85 \\
            \midrule
            \multirow{6}{*}{SVHN} & {\IT}                &  \gray{86.63}                   \\
                                    & Hyperopt            & 67.20                     \\
                                    %& Opentuner           & 91.79                    \\
                                    & Random       & \gray{91.86}                              \\
                                    & TuRBO               &  80.41                             \\
                                    & HEBO                & 92.89                              \\
                                    & BOHB                & 79.67 \\
            \midrule
            \multicolumn{3}{c}{\textbf{Bayesmark}} \\
            \midrule
            \textbf{Dataset}                 & \textbf{HPO method} & \multicolumn{1}{l}{\textbf{Score}} \\
            \midrule
            \multirow{6}{*}{breast} & {\IT}          & \gray{93.98}                     \\
                                    & Hyperopt            & \gray{92.68}                     \\
                                    %& Opentuner           & \gray{92.40}                     \\
                                    & Random       & 90.45                              \\
                                    & TuRBO               & 88.97                              \\
                                    & HEBO                & 90.84                              \\
                                    & BOHB                & 92.37 \\
            \midrule
            \multirow{6}{*}{digits} & {\IT}          & 84.74                                  \\
                                    & Hyperopt            & \gray{96.85}                     \\
                                    %& Opentuner           & 75.76                              \\
                                    & Random       & 87.39                              \\
                                    & TuRBO               & 89.51                              \\
                                    & HEBO                & \gray{95.24}                     \\
                                    & BOHB                & 91.50 \\
            \midrule
            \multirow{6}{*}{iris}   & {\IT}          & \gray{91.00}                    \\
                                    & Hyperopt            & 79.83                              \\
                                    %& Opentuner           & 85.06                     \\
                                    & Random       & 82.98                              \\
                                    & TuRBO               & 83.52                              \\
                                    & HEBO                & 78.27                              \\ 
                                    & BOHB                & \gray{92.30} \\
            \midrule
            \multirow{6}{*}{wine}   & {\IT}          & \gray{92.35}                     \\
                                    & Hyperopt            & 83.98                              \\
                                    %& Opentuner           & 83.43                              \\
                                    & Random       & 76.59                              \\
                                    & TuRBO               & 81.15                              \\
                                    & HEBO                & 74.93                              \\   
                                    & BOHB                & 81.68
                                            \\
            % \midrule
            % \multirow{7}{*}{diabetes**} & {\IT}        & \gray{98.79}                                                   \\
            %                         & Hyperopt            & \gray{99.16}                              \\  
            %                         & Opentuner           & \gray{99.04}                              \\            
            %                         & Random       & 97.16                              \\                                      
            %                         & TuRBO               & \gray{98.67}                              \\  
            %                         & HEBO                & 98.09                              \\ 
            %                         & BOHB                & 97.86
            %                                 \\
            \bottomrule
        \end{tabular}
    \end{subtable}
    \hspace{.03\textwidth}
    \begin{subtable}[t]{.45\textwidth}
        \begin{tabular}[t]{llr}
            \toprule
            \multicolumn{3}{c}{\textbf{OpenML}} \\
            \midrule
            \textbf{Dataset}                 & \textbf{HPO method} & \multicolumn{1}{l}{\textbf{AUC}} \\
            \midrule
            \multirow{6}{*}{\href{https://openml.org/t/53}{vehicle}} & {\IT}                & \gray{0.885}               \\
                                    & Hyperopt            &  \gray{0.883}                   \\
                                    %& Opentuner           &                     \\
                                    & Random       & 0.873                              \\
                                    & TuRBO               & \gray{0.882}                                \\
                                    & HEBO                & \gray{0.884}                              \\
                                    & BOHB                & \gray{0.883} \\
            \midrule
            \multirow{6}{*}{\href{https://openml.org/t/10101}{blood-transf...}} & {\IT}                &  0.721               \\
                                    & Hyperopt            &  \gray{0.728}                   \\
                                    %& Opentuner           &                      \\
                                    & Random       &  0.708                             \\
                                    & TuRBO               &  0.720                            \\
                                    & HEBO                & 0.718                              \\
                                    & BOHB                & \gray{0.725} \\
            \midrule
            \multirow{6}{*}{\href{https://openml.org/t/146818}{Australian}} & {\IT}                &  \gray{0.934}                  \\
                                    & Hyperopt            &  \gray{0.932}                   \\
                                    %& Opentuner           &                      \\
                                    & Random       &  0.928                             \\
                                    & TuRBO               &  \gray{0.932}                            \\
                                    & HEBO                & \gray{0.935}                              \\
                                    & BOHB                & \gray{0.928} \\
            \midrule
            \multirow{6}{*}{\href{https://openml.org/t/146821}{car}} & {\IT}                &  1.0                  \\
                                    & Hyperopt            &  \gray{1.0}                   \\
                                    %& Opentuner           &                      \\
                                    & Random       & 1.0                              \\
                                    & TuRBO               & 1.0                              \\
                                    & HEBO                & 1.0                              \\
                                    & BOHB                & \gray{1.0} \\
            \midrule
            \multirow{6}{*}{\href{https://openml.org/t/9952}{phoneme}} & {\IT}                & 0.560                  \\
                                    & Hyperopt            & \gray{0.564}                    \\
                                    %& Opentuner           &                      \\
                                    & Random       & 0.561                              \\
                                    & TuRBO               & \gray{0.564}                              \\
                                    & HEBO                & 0.563                              \\
                                    & BOHB                & \gray{0.563} \\
            \midrule
            \multirow{6}{*}{\href{https://openml.org/t/146822}{segment}} & {\IT}                &  \gray{0.961}                 \\
                                    & Hyperopt            & \gray{0.960}                    \\
                                    %& Opentuner           &                      \\
                                    & Random       &  0.955                             \\
                                    & TuRBO               & \gray{0.960}                              \\
                                    & HEBO                & \gray{0.961}                             \\
                                    & BOHB                & \gray{0.960} \\
            \midrule
            \multirow{6}{*}{\href{https://openml.org/t/31}{credit-g}} & {\IT}                & 0.778                  \\
                                    & Hyperopt            &  \gray{0.782}                   \\
                                    %& Opentuner           &                      \\
                                    & Random       &  0.766                             \\
                                    & TuRBO               &  0.763                             \\
                                    & HEBO                & 0.763                              \\
                                    & BOHB                & \gray{0.752} \\
            \midrule
            \multirow{6}{*}{\href{https://openml.org/t/3917}{kcl}} & {\IT}                & \gray{0.816}                  \\
                                    & Hyperopt            &  \gray{0.817}                   \\
                                    %& Opentuner           &                      \\
                                    & Random       & 0.769                              \\
                                    & TuRBO               & 0.775                              \\
                                    & HEBO                & \gray{0.816}                              \\
                                    & BOHB                & \gray{0.785} \\
            \bottomrule
            \end{tabular}
    \end{subtable}
    }
\end{table*}
%\vspace{-10pt}
%\input{results/image}

We compare our approach based on strong convexity (which we call {\IT}) with other popular hyper-parameter optimization algorithms. We randomly sample 50 configurations, compute their strong convexity, and run the top 10, reporting the best-performing one. We repeat all experiments 20 times, and compare results using pairwise Mann-Whitney tests \cite{mann1947test} with a Benjamini-Hochberg correction procedure for p-values (as endorsed by \citet{farcomeni2008review}), employing a 5\% significance level.

For tabular datasets, we experiment on the Bayesmark datasets used in the NeurIPS 2020 Black-Box Optimization Challenge and the 8 datasets used in the MLP benchmarks in HPOBench \cite{eggensperger2021hpobench}. For convolutional networks, we run experiments on MNIST \cite{lecun1998mnist} and SVHN \cite{netzer2011reading}.

For Bayesmark, we use the default set of hyper-parameters for MLPs, which has a size of 134M\footnote{\url{https://github.com/uber/bayesmark/blob/master/bayesmark/sklearn_funcs.py}}. For MNIST, we used Conv - MaxPooling blocks, followed by two fully-connected layers. For SVHN, we used Conv - BatchNorm - Conv - MaxPooling - Dropout layers, followed by a fully-connected layer, a dropout layer, and a final fully-connected layer. The range of hyper-parameters for all these models is shown in Table \ref{tab:hpo_space}.

We use the metrics employed by prior work to ensure a fair and consistent evaluation. For Bayesmark datasets, we report the mean normalized score, which first calculates the performance gap between observations and the global optimum and divides it by the gap between random search and the optimum. For MNIST and SVHN, we use the accuracy score. For the OpenML datasets, we use the area under the ROC curve, since we found that many of them had notable class imbalances.

\begin{table}[t]
    \centering
    \caption{Hyper-parameters used in this study. Ranges are inclusive. Unless specified, the ranges are linear. For Bayesmark, we use the default hyper-parameter set, whose size is 134M.}
    \label{tab:hpo_space}
    \begin{tabular}{lr}
        \toprule
        \multicolumn{2}{c}{OpenML} \\
        \midrule
        \textbf{Hyper-parameter} & \textbf{Range} \\
        \midrule
        Network depth & (1, 4) \\
        Network width & (16, 1024), $\log_2$ \\
        Batch size & (4, 256), $\log_2$ \\
        Initial learning rate & $(10^{-5}, 1.0)$, $\log_{10}$ \\
        \midrule
        \multicolumn{2}{c}{MNIST} \\
        \midrule
        \textbf{Hyper-parameter} & \textbf{Range} \\
        \midrule
        Number of filters & (2, 6) \\
        Kernel size & (2, 6) \\
        Padding & Valid, same \\
        Number of conv blocks & (1, 3) \\
        \midrule
        \multicolumn{2}{c}{SVHN} \\
        \midrule
        \textbf{Hyper-parameter} & \textbf{Range} \\
        \midrule
        Number of filters & (2, 6) \\
        Kernel size & (2, 6) \\
        Padding & Valid, same \\
        Number of conv blocks & (1, 3) \\
        Dropout rate & (0.2, 0.5) \\
        Final dropout rate & (0.2, 0.5) \\
        Number of units & (32, 512), $\log_2$ \\
        \bottomrule
    \end{tabular}
\end{table}

Our results, shown in Table \ref{tab:results}, demonstrate that strong convexity is both capable of achieving strong performance on most datasets. Importantly, these are also computationally cheap (see Section \ref{sec:runtime}).

\subsection{Runtime}
\label{sec:runtime}

In practice, computing the strong convexity is cheap. In detail, we train the network with the given hyper-parameters for one epoch and then use the equations derived to compute the strong convexity in mini-batches. The one epoch of training moves the network weights closer to the final position, so that the measured loss criterion is more accurate than from a randomly initialized point. 

\begin{figure}
    \centering
    \includegraphics[width=.4\linewidth]{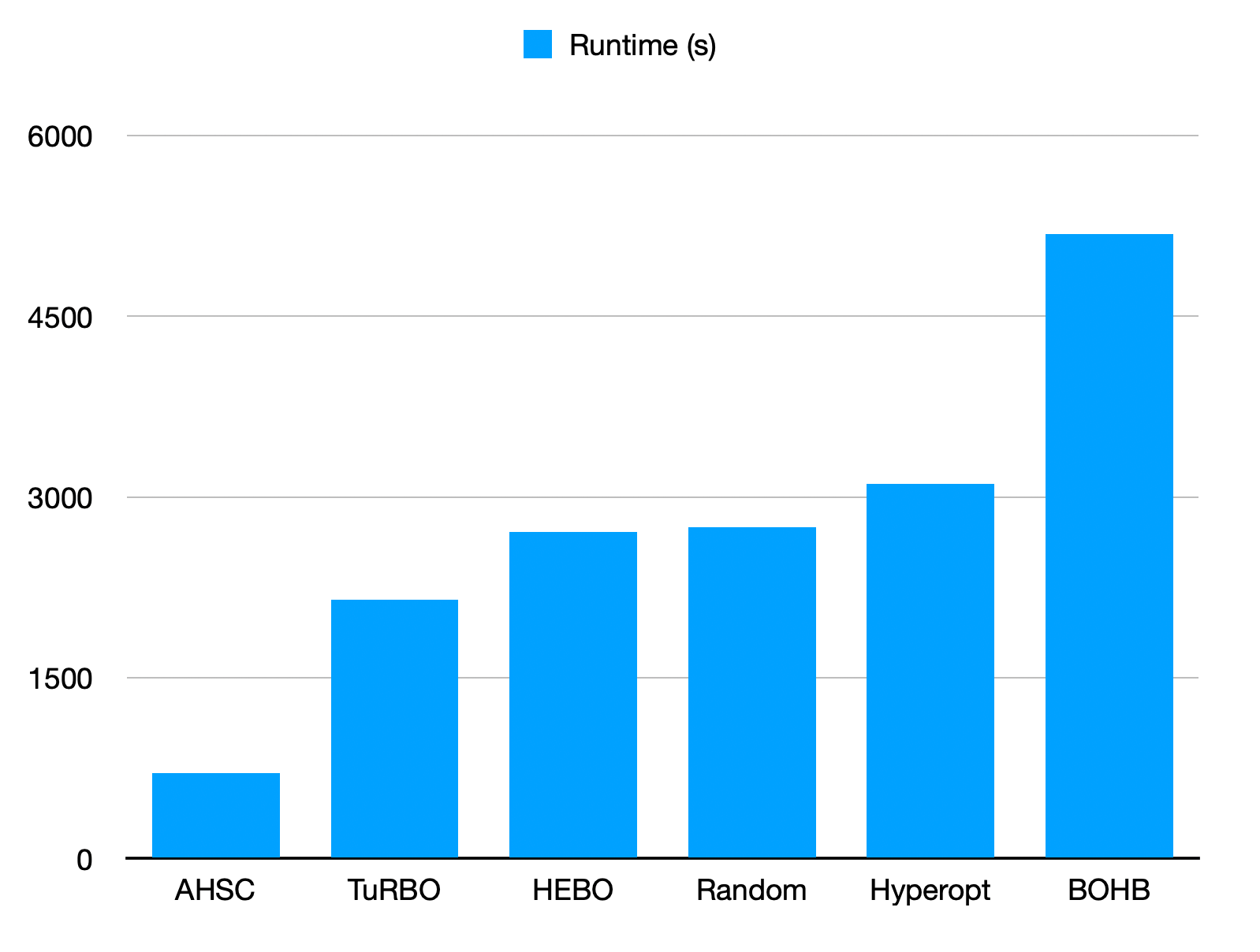}
    \caption{Algorithm runtimes on the vehicle dataset.}
    \label{fig:runtimes}
\end{figure}

On a machine with an Intel Cascade Lake CPU with 4 vCPUs and 23GB RAM and no GPU, where we ran our NeurIPS Black-Box Optimization Challenge experiments, we measured the cost of computing the strong convexity over 15 runs with varying batch sizes. The median number of batches was 15.51 (442 samples), which took a median of 0.47 seconds. Therefore, it takes 0.03s/batch/config to compute the strong convexity. For example, the breast cancer dataset has 569 samples. Using the mean batch size in the hyper-parameter space of 130, that evaluates to 4.38 batches, which we expect to take $0.03 \times 4.38 \times 50 = 6.57$s to compute the strong convexity for 50 configurations.

The above experiments suggest that computing the strong convexity this way is computationally cheap, since we train for the full epochs only for the top 10 configurations. Figure \ref{fig:runtimes} shows the runtimes for each algorithm on the vehicle dataset, on a machine with an RTX 2060 Super. Our approach requires between 13\% to 33\% of the runtime compared to other algorithms.

\section{Conclusion and Future Work}

In this paper, we developed a novel white-box hyper-parameter optimization algorithm that, after some cheap computation, requires only 10 full runs to find a good configuration. We demonstrated our results on 12 tabular and two image datasets. 

It has not escaped our attention that this method does not allow the user to specify a preference for evaluation metrics such as recall or precision. We leave this as future work. In particular, we exploit the fact that the Pareto frontier is a subset of the convex hull of the hyper-parameter performance scores. To find configurations that do well on some metric, we traverse the Pareto frontier and compute quantized\footnote{This is only necessary if the space is not dense in $\mathbb{R}^d$; for example, $Z$ is nowhere dense in $\mathbb{R}$ so that for a space $\mathbb{R}^n \times \mathbb{Z}$, quantization would be necessary.} convex combinations of adjacent points and also test them. This is similar to the approach of \citet{ammar2004multi}. However, this approach adds additional computational cost.

Our hyper-parameter optimization method has two key limitations: first, it is limited to learners for which a loss function can be defined. In some cases such as Naive Bayes, a surrogate such as the negative log-likelihood can be used, for which the strong convexity can be computed. Even in cases where the loss is not twice-differentiable, one can use a finite difference approximation to compute the Hessian (see \citet{nocedal1999numerical}):
\begin{equation*}
%\scalebox{0.85}{$
\frac{\partial^2 f}{\partial x_i \partial x_j}(x) \approx \frac{1}{\epsilon^2}\left( f(x + \epsilon e_i + \epsilon e_j) - f(x + \epsilon e_i) - f(x + \epsilon e_j) + f(x) \right)
%$}
\end{equation*}

where the error is $\mathcal{O}(\epsilon)$. The second, potentially more important limitation is that the strong convexity cannot be compared across learners, especially if different losses are used. For example, while algorithms such as TPE can be used on hyper-parameter spaces with multiple classes of learners, our approach cannot: the entire hyper-parameter space must have comparable strong convexity values, for which the same class of learners (such as neural networks, Naive Bayes, logistic regression, etc.) must be used. However, this can be resolved by using the same loss function across learning algorithms. For example, a negative log-likelihood ratio loss has been proposed for neural classifiers \cite{yao2020negative}, which is compatible with other learners.

\bibliographystyle{plainnat}
\bibliography{main}

\newpage
\appendix

\section{Related Work}
\label{sec:app:lit}

\textbf{Hyper-parameter optimization.} There is significant prior work in hyper-parameter optimization \cite{agrawal2019dodge, cowen2022hebo, li2017hyperband, bergstra2011algorithms, bergstra2012random, falkner2018bohb, eriksson2019scalable, ansel:pact:2014}. Indeed, as learning systems become more intricate, it is crucial that we eke out the most performance. However, this is a non-trivial problem, as evidenced by the long line of research in this direction. 

The simplest form of hyper-parameter search is random search, which tries $n$ randomly chosen hyper-parameter configurations. Opentuner \cite{ansel2014opentuner} is a multi-armed bandit meta-technique with a sliding window that incorporates an exploration/exploitation trade-off based on the number of times a specific technique is used. It combines DE, a greedy bandit mutation technique, and hill-climbing methods. 

Bayesian Optimization (BO) has emerged as the most popular technique for HPO. \citet{bergstra2011algorithms} propose the Tree of Parzen Estimators (TPE) algorithm. Rather than model $p(y|x)$, TPE models $p(x|y)$ as 
$$
p(x|y) = \begin{cases}
    l(x) &  y < y^* \\
    g(x) &  y \geq y^*
    \end{cases}
$$
where $y^*$ is chosen so that $p(y < y^*) = \gamma$ for some quantile $\gamma$. The functions $l(x)$ and $g(x)$ are kernel density estimates. TPE optimizes the EI, which they show is equivalent to maximizing $l(x) / g(x)$.
\citet{snoek2012practical} use Gaussian Process (GP) models as the surrogate function in Bayesian optimization. They use Expected Improvement (EI) as the acquisition function. Similarly, \citet{hernandez2014predictive} use predictive entropy search (PES) as the acquisition function. \citet{swersky2014freeze} exploit iterative training procedures in their Bayesian optimization framework, which they call freeze-thaw Bayesian optimization. \citet{snoek2015scalable} use neural networks for modeling distributions over functions that yields an approach that scales linearly over data size (rather than cubically as in GP-based Bayesian optimization). BOHB \cite{falkner2018bohb} combines the BO-based TPE with HyperBand \cite{li2017hyperband}, replacing the initial random configurations with a model-based search. Notably, BOHB uses a single multi-dimensional KDE instead of hierarchical single-dimensional KDEs used by TPE. The authors of HEBO \cite{cowen2022hebo} note that (i) even simple HPO problems can be non-stationary and heteroscedastic (ii) different acquisition functions can conflict. To tackle the former, they use the Box-Cox \cite{box1964analysis} and Yeo-Johnson \cite{yeo2000new} output transformations and the Kumaraswamy \cite{kumaraswamy1980generalized} input transformation. It also uses NSGA-II to optimize a multi-objective acquisition function. TuRBO \cite{dou2023turbo} assumes the hyper-parameter to performance mapping is Lipschitz, and generates pseudo-points to improve convergence of vanilla BO. It also uses an ensemble of learners to predict performance, and if the prediction is poor, uses it to instead update the GP model. Finally, we mention PriorBand \cite{mallik2023priorband}, which incorporates an expert's prior beliefs about good configurations, but maintains good performance even if that prior is bad.

We defer to \cite{feurer2019hyperparameter} and \cite{bischl2023hyperparameter} for recent reviews on hyper-parameter optimization techniques. There is also a long line of work studying neural architecture search, which aims to find optimal architectures for a dataset. We refer the reader to \citet{white2023neural} for a comprehensive review of the field.

% There is also a long (and growing) line of work studying \textit{neural architecture search} (\cite{baker2016designing,elsken2017simple,liu2017hierarchical,liu2018darts,stanley2002evolving,zoph2018learning,xie2018snas,real2019regularized}). This field studies systematic approaches to find optimal neural network architectures for a dataset. For example, \cite{liu2017hierarchical} start with small building blocks and work up to a hierarchical, more complex architecture. \cite{liu2018darts} frame the problem in a way that allows them to use gradient descent to find an optimal architecture. We do not use neural architecture search in this work, but mention it due to its relevance to hyper-parameter optimization; indeed, neural architecture search focuses on one specific subset of hyper-parameters--the network architecture--without focusing on others such as batch size, optimizer-specific hyper-parameters such as learning rate, preprocessing algorithms used, etc. Moreover, neural architecture search approaches can be slow and take days on a modern GPU to find an optimal architecture, for one dataset and for one run. Our methods use statistical analysis to check that our approach is better by a statistically significant margin, and therefore neural architecture search could not be incorporated in our overall framework. Nevertheless, we direct the reader to \cite{elsken2019neural} for a recent literature review of neural architecture search approaches.

Several benchmarks have been proposed for hyper-parameter optimization: notable ones include YAHPO Gym \cite{pfisterer2022yahpo}, HPO-B \cite{arango2021hpo}, and HPOBench \cite{eggensperger2021hpobench}.

\textbf{Flat minima and generalization.} The idea of flat minima was first studied by \citet{hochreiter1994simplifying}. In particular, they define ``flat minima'' as large connected regions where the weights are $\epsilon-$optimal. \citet{hochreiter1997flat} intuit that because sharper minima require higher precision, flatter minima require less bits to describe. They use this intuition to show that flat minima correspond to minimizing the number of bits required to describe the weights of a neural network. This notion of minimum description length (MDL) \cite{rissanen1983universal, grunwald2007minimum} was also exploited early on by \citet{hinton1993keeping}. Flat minima were revisited by \citet{chaudhari2019entropy}, who noted that minima with low generalization error have a large proportion of their eigenvalues close to zero. They then construct a modified Gibbs distribution corresponding to an energy landscape $f$, and minimize the negative local entropy of this modified distribution, and approximate the gradient via stochastic gradient Langevin dynamics (SGLD) \cite{welling2011bayesian}. However, their assumptions were, admittedly unrealistic. \citet{keskar2016large} show that when using large batch sizes, optimizers converge to sharp minima, which are characterized by many large positive eigenvalues of the Hessian. Further, they define the notion of sharpness as a generalization measure as the robustness to adversarial perturbations in the parameter space:
\begin{equation}
\zeta(\boldsymbol w; \epsilon) = \frac{\max\limits_{\abs{\boldsymbol v} \leq \epsilon(\abs{\boldsymbol w} + 1)} f(\boldsymbol w + \boldsymbol v) - f(\boldsymbol w)}{1 + f(\boldsymbol w)}
\label{eq:sharpness}
\end{equation}
and compute this using 10 iterations of L-BFGS-B \cite{byrd1995limited} with $\epsilon = \{ 10^{-3}, 5 \cdot 10^{-4} \}$. In particular, the above is closely related to the largest eigenvalue of $\nabla^2 f(\boldsymbol w)$. This notion of sharpness was also endorsed by \citet{jiang2019fantastic}, who performed a large-scale study of many complexity measures on two datasets, with 2,187 convolutional networks. In particular, they endorse the following metrics for generalization: (i) variance of gradients (ii) squared ratio of magnitude of parameters to magnitude of perturbation, à la \citet{keskar2016large} (iii) path norm \cite{neyshabur2017exploring} (iv) VC-dimension (inversely correlated).
    
A line of work in physics \cite{baldassi2015subdominant,baldassi2016unreasonable} showed that in the discrete weight scenario (a much more difficult problem), isolated minima were rare, but there existed accessible, dense regions of subdominant minima, and that these were robust to perturbations and generalized better. These authors devised algorithms explicitly designed to search for nonisolated minima. In the continuous weight space, nonisolated minima correspond to flat minima. \citet{dziugaite2017computing} obtain nonvacuous generalization bounds for deep overparameterized neural networks using the PAC-Bayes framework \cite{mcallester1999pac}. \citet{neyshabur2014search} showed that increasing the number of hidden units (which in turn, increases the number of trainable parameters) can lead to a decrease in generalization error with the same training error. \citet{neyshabur2017exploring} showed that sharpness as computed by \eqref{eq:sharpness} is not sufficient to capture the generalization behavior (but noting that ``combined with the norm, sharpness does seem to provide a capacity measure''), and advocate for expected sharpness in the PAC-Bayesian framework, similar to \citet{dziugaite2017computing}. They show that plots of expected sharpness versus KL divergence in PAC-Bayes bounds for varying dataset sizes capture generalization well. \citet{li2018visualizing} showed that the sharpness of the loss surface correlates well with the generalization error. In their seminal paper, \citet{jastrzkebski2017three} showed that SGD is a Euler-Maruyama discretization of a stochastic differential equation whose dynamics are influenced by the ratio of learning rate to batch size (which they call ``stochastic noise''), and that SGD finds wider minima with higher stochastic noise levels than sharper minima. \citet{wu2023implicit} study the flat minima hypothesis through the lens of dynamical stability, and show that SGD will escape from overly sharp (measured by the Frobenius norm of the Hessian), low-loss areas exponentially fast. We note that they use the associate empirical Fisher matrix (AEFM) as an approximation for the Hessian, which holds for low empirical risk (and converges to the Hessian, see \citet{kunstner2019limitations}).

In search of flatter loss surfaces, \citet{seong2018towards} propose the use of non-monotonic learning rate schedules. They advocate for large learning rates, which enable the optimization algorithm to escape sharp minima, and descend into flatter minima. The seminal works of \citet{dauphin2014identifying} and \citet{choromanska2015loss} showed both theoretically and empirically that local minima are more likely to be located close to the global minimum. In particular, \citet{dauphin2014identifying} showed using the perspectives of random matrix theory (via the eigenvalue distribution of Gaussian random matrices \cite{wigner1958distribution}), statistical physics (via the analysis of critical points in Gaussian fields by \citet{bray2007statistics}), and neural network theory \cite{saxe2013exact} that saddle points are exponentially less likely than local minima.

Most recently, \citet{wen2023sharpness} showed that sharpness is neither necessary, nor sufficient for generalization, by studying some simple architectures, showing that generalization depends on the data distribution as well as the architecture; for example, merely adding a bias to a 2-layer MLP makes generalization impossible for the XOR dataset.
    
Of course, it is not possible to discuss generalization in deep learning without discussing the results of \citet{zhang2021understanding}, who showed that deep learners can fit with zero training error on random labels using an architecture that generalizes well when fit to the correct labels. \citet{bartlett2019nearly} and \citet{harvey2017nearly} found that the VC-dimension for deep ReLU networks is $\mathcal{O}(WL \log W)$ and $\Theta(WU)$ respectively, where $W$ is the number of parameters, $L$ the number of layers, and $U$ the number of units.  \citet{hardt2016train} show that stochastic gradient methods are uniformly stable\footnote{An algorithm $A$ is $\epsilon-$uniformly stable if $\forall S, S^\prime \in Z$ for some space $Z$, such that the datasets $S$ and $S^\prime$ differ by at most one example, $\sup\limits_z \mathbb{E}_A [f(A(S); z) - f(A(S^\prime); z)] \leq \epsilon$}, which implies generalization \textit{in expectation}.
\section{Background}
\label{sec:app:background}

% \begin{definition}[Sub-differential]
%     For a convex function $f: \mathcal{X} \to \mathbb{R}^*$, the sub-differential $\partial f(x)$ is defined as 
%     $$
%     \partial f(x) = \{ y \in \mathcal{X}: \forall z, f(x + z) \geq f(x) + \langle y, z \rangle \}
%     $$
% \end{definition}

\begin{definition}[Fenchel conjugate]
    The Fenchel conjugate, $f^*: \mathcal{X} \to \mathbb{R}^*$ for a convex function $f: \mathcal{X} \to \mathbb{R}^*$ is defined as $f^*(y) = \sup\limits_{x \in \mathcal{X}} \langle x, y \rangle - f(x)$
\end{definition}

\begin{definition}[Dual norm]
    Given a norm $\lVert \cdot \rVert$ on $\mathcal{X}$, the dual norm $\lVert \cdot \rVert_*$ is defined as 
    $$
    \lVert y \rVert_* = \sup \{ \langle x, y \rangle: \lVert x \rVert \leq 1 \}
    $$
\end{definition}

Note that the Fenchel conjugate of $\frac{1}{2} \lVert x \rVert$ is $\frac{1}{2} \lVert y \rVert_*$.

\begin{definition}[Strong convexity]
    A function $f: \mathcal{X} \to \mathbb{R}^*$ is $\mu-$strongly convex with respect to $\lVert \cdot \rVert$ if $\forall x, y$ in the relative interior of $\text{dom } f$ and $\alpha \in (0, 1)$,
    $$
    f(\alpha x + (1 - \alpha) y) \leq \alpha f(x) + (1 - \alpha) f(y) - \frac{1}{2} \mu \alpha (1 - \alpha) \lVert x - y \rVert^2
    $$
\end{definition}

\begin{definition}[Smoothness]
    A function $f: \mathcal{X} \to \mathbb{R}$ is $\beta-$smooth with respect to $\lVert \cdot \rVert$ if $f \in C^1$ and if $\forall x, y \in \text{dom } f$,
    $$
    f(x + y) \leq f(x) + \langle \nabla f(x), y \rangle + \frac{1}{2} \beta \lVert y \rVert^2
    $$
\end{definition}

% \begin{definition}[Dense set]
%     A subset $A$ of a topological space $X$ is said to be \textit{dense} in $X$ if every item in $X$ is present in or is arbitrarily close to some point in $A$.
% \end{definition}

% \begin{definition}(Lebesgue integral)
%     The Lebesgue integral is a generalization of the Riemann integral. It is defined in terms of the Lebesgue sum
%     \[
%         \sum_i \alpha_i \mu(S_i)
%     \]
%     where $\alpha_i$ is the value of the function in the subinterval $i$, and $\mu(S_i)$ is the Lebesgue measure of the interval $S_i$ of points for which the value of the function is approximately $\alpha_i$. A function is \textit{Lebesgue integrable} if its Lebesgue integral is finite.
% \end{definition}

% \begin{definition}[$L^p$ space]
%     An $L^p$ space is a space of measurable functions for which the $p$th power is Lebesgue integrable.
% \end{definition}
\section{Auxiliary Proofs}
\label{sec:proofs}

\begin{lemma}
    Let $f: \mathbb{R}^d \to \mathbb{R}$ be a differentiable function. Then, $\mu-$strong convexity implies:
    \begin{enumerate}[(i)]
        \item (Polyak-Łojasiewicz (PL) inequality)
        \[
            \frac{1}{2} \norm{ \nabla f(x) }^2 \geq \mu (f(x) - f(x^*))
        \]

        \item 
        \[
            \norm{\nabla f(x) - \nabla f(y)} \geq \mu \norm{x - y}
        \]

        \item 
        \[
            \left( \nabla f(x) - \nabla f(y) \right)^T (x - y) \leq \frac{1}{\mu} \norm{\nabla f(x) - \nabla f(y)}^2
        \]
    \end{enumerate}
    \label{lemma:stcvx}
\end{lemma}
\begin{proof}
    \begin{enumerate}[(i)]
        \item Strong convexity implies
        $$
        f(y) \geq f(x) + \nabla f(x)^T (y - x) + \frac{\mu}{2} \norm{y - x}^2
        $$
        Minimizing with respect to $y$ yields the result.

        \item Strong convexity gives us:
        \begin{equation}
        \left( \nabla f(x) - \nabla f(y) \right)^T (x - y) \geq \mu \norm{x - y}^2 \label{eq:proof:l1}
        \end{equation}

        Applying Cauchy-Schwarz inequality gives us:
        \begin{equation*}
            \scalebox{0.9}{$
                \begin{aligned}
                    \norm{\nabla f(x) - \nabla f(y)}\norm{x - y} &\geq \left( \nabla f(x) - \nabla f(y) \right)^T (x - y) \\
                    &\geq \mu \norm{x - y}^2
                \end{aligned}
            $}
        \end{equation*}
        where the last step comes from \eqref{eq:proof:l1}.

        \item Set $\phi_x(z) = f(z) - \nabla f(x)^T z$. It is easy to see that $\phi_x$ is also $\mu-$strongly convex. Applying the Polyak-Łojasiewicz inequality to $\phi_x(z)$ with $z^* = x$,
        \begin{equation}
            \scalebox{0.8}{$
                \begin{aligned}
                    \left( f(y) - \nabla f(x)^T y \right) - \left( f(x) - \nabla f(x)^T x \right) &= \phi_x(y) - \phi_x(z^*) \nonumber \\
                    &\leq \frac{1}{2\mu} \norm{\nabla \phi_x(y)}^2 \nonumber \\
                    &\leq \frac{1}{2\mu} \norm{\nabla f(y) - \nabla f(x)}^2 \label{eq:proof:l2}
                \end{aligned}
            $}
        \end{equation}
        Swapping $x$ and $y$ in the above,
        \begin{equation}
            \scalebox{0.8}{$
                \begin{aligned}
                    \left( f(x) - \nabla f(y)^T x \right) - \left( f(y) - \nabla f(y)^T y \right) \leq \frac{1}{2\mu} \norm{\nabla f(x) - \nabla f(y)}^2 \label{eq:proof:l3}
                \end{aligned}
            $}
        \end{equation}
        Adding \eqref{eq:proof:l2} and \eqref{eq:proof:l3} yields the result.
    \end{enumerate}
\end{proof}

\begin{lemma}
    If $f: \mathbb{R}^n \to \mathbb{R}$ is smooth and $\mu-$strongly convex, then
    \[
        \frac{1}{2\mu} \norm{\nabla f(x)}^2 \geq f(x) - f(x^*) \geq \frac{\mu}{2} \norm{x - x^*}^2
    \]
    \label{lemma:pl-conseq}
\end{lemma}
\begin{proof}
    The first part is the Polyak-Łojasiewicz inequality. The second part follows from the definition of strong convexity and setting $y = x, x = x^*$ and using $f(x^*) \geq \min\limits_y f(y)$.
\end{proof}

\begin{theorem}[Smooth and strongly convex gradient descent]
    Suppose $f: \mathbb{R}^n \to \mathbb{R}$ be $\beta-$smooth and $\mu-$strongly convex. Then with the gradient descent update rule
    \[
        x_{k+1} = x_k - \frac{1}{\beta} \nabla f(x_k)
    \]
    where $\frac{1}{\beta}$ is the learning rate, we have
    \[
        f(x_k) - f(x^*) \leq \left( 1 - \frac{\mu}{\beta} \right)^k (f(x_0) - f(x^*))
    \]
    Consequently, we require $\frac{\beta}{\mu}\log \frac{f(x_0) - f(x^*)}{\epsilon}$ iterations to find an $\epsilon-$optimal point.
    \label{lemma:smoothgd}
\end{theorem}
\begin{proof}
    From the above results,
    \[
        f(x_{k+1}) \leq f(x_k) - \frac{1}{2\beta} \norm{\nabla f(x_k)}^2
    \]
    and Lemma \ref{lemma:pl-conseq} gives us
    \[
        \norm{\nabla f(x_k)}^2 \geq 2\mu (f(x_k) - f(x^*))
    \]
    Therefore,
    \begin{align*}
        f(x_{k+1}) - f(x^*) &\leq f(x_k) - f(x^*) - \frac{1}{2\beta} \norm{\nabla f(x_k)}^2 \\
        &\leq f(x_k) - f(x^*) - \frac{\mu}{\beta} (f(x_k) - f(x^*)) \\
        &= \left( 1 - \frac{\mu}{\beta} \right) (f(x_k) - f(x^*)) \\
        f(x_k) - f(x^*) &\leq \left( 1 - \frac{\mu}{\beta} \right)^k (f(x_k) - f(x^*)) \\
        &\leq \exp\left(-\frac{\mu k}{\beta}\right) (f(x_0) - f(x^*)) 
    \end{align*}
    Therefore we need $k = \frac{\beta}{\mu}\log \frac{f(x_0) - f(x^*)}{\epsilon}$ iterations for $\epsilon-$optimality.
\end{proof}

Note that strong convexity guarantees optimality. $\beta-$smoothness can only assure $\epsilon-$criticality. This implies the existence of global minima. 
%Moreover, note that since $\log \cosh$ is strongly convex (see Theorem \ref{th:logcosh:stcvx}) and $\beta-$smooth (cf. Theorem \ref{th:logcosh:beta}), it is particularly useful in the regression setting.

Strong convexity is a necessary condition for learnability, as along with $\beta-$smoothness, it can be shown that such problems are learnable \cite{shalev2014understanding}. Additionally, strong convexity provides a quadratic lower bound on the growth of the loss function, which implies that the convexity condition will never be violated in the local domain of the function in the context of deep regression with regularization.

\begin{theorem}
    Suppose $f$ is a closed and convex function. Then $f$ is $\beta-$strongly convex with respect to a norm $\lVert \cdot \rVert$ iff $f^*$ is $\frac{1}{\beta}-$smooth with respect to the dual norm $\lVert \cdot \rVert_*$. That is, if $f^*$ is smooth, then $f$ is strongly convex.
\end{theorem}
\begin{proof}
    We defer the proofs to \citet{shalev2007primal}, Lemma 15 ($1 \Rightarrow 2$) and \citet{kakade2009duality}, Theorem 6 ($2 \Rightarrow 1$).
\end{proof}
The above result has implications for the generalization bounds of various algorithms, such as lasso regression (which is a special case of the group lasso algorithm discussed in \cite{kakade2009duality}), kernel learning, and online control. For a detailed exposition, we refer the reader to \citet{kakade2009duality}.

\end{document}